\documentclass[twoside]{article}
\usepackage{times}  % DO NOT CHANGE THIS
\usepackage{helvet} % DO NOT CHANGE THIS
\usepackage{courier}  % DO NOT CHANGE THIS
\usepackage[hyphens]{url}  % DO NOT CHANGE THIS
\usepackage{graphicx} % DO NOT CHANGE THIS
\usepackage{natbib}  % DO NOT CHANGE THIS AND DO NOT ADD ANY OPTIONS TO IT
\usepackage{caption} % DO NOT CHANGE THIS AND DO NOT ADD ANY OPTIONS TO IT
\usepackage[utf8]{inputenc} % allow utf-8 input
\usepackage[T1]{fontenc}    % use 8-bit T1 fonts
\usepackage{url}            % simple URL typesetting
\usepackage{booktabs}       % professional-quality tables
\usepackage{amsfonts}       % blackboard math symbols
\usepackage{nicefrac}       % compact symbols for 1/2, etc.
\usepackage{microtype}      % microtypography
\usepackage{bm}
\usepackage{amsmath}
\usepackage{amssymb}
\usepackage{amsthm}
\usepackage{xcolor}
\usepackage{graphicx}
\usepackage{floatrow}
\usepackage{subcaption}
\usepackage{bbold}
\usepackage{algorithm}
\usepackage{algorithmic}
\newtheorem{theorem}{{\bf Theorem}}

\newtheorem{remark}{{\bf Remark}}

\usepackage{algorithm} 
\usepackage{array}
\usepackage{wrapfig}
\usepackage{multirow}
\usepackage{tabularx}
\usepackage[accepted]{aistats2021}
\begin{document}


% \runningauthor{Surname 1, Surname 2, Surname 3, ...., Surname n}

\twocolumn[

\author{
	Khalil Elkhalil$^1$, Ali Hasan$^2$, Jie Ding$^3$, Sina Farsiu$^{1,2}$, Vahid Tarokh$^1$ \\
	$^1$Department of Electrical and Computer Engineering,
	Duke University,
	Durham, NC 27701\\
	$^2$Department of Biomedical Engineering,
	Duke University,
	Durham, NC 27701\\
		\texttt{\{khalil.elkhalil,ali.hasan,sina.farsiu,vahid.tarokh\}@duke.edu} \\
	$^3$School of Statistics,
University of Minnesota 
Minneapolis, MN 55455  \\
  dingj@umn.edu \\
	%% examples of more authors
	%% \And
	%% Coauthor \\
	%% Affiliation \\
	%% Address \\
	%% \texttt{email} \\
	%% \AND
	%% Coauthor \\
	%% Affiliation \\
	%% Address \\
	%% \texttt{email} \\
	%% \And
	%% Coauthor \\
	%% Affiliation \\
	%% Address \\
	%% \texttt{email} \\
	%% \And
	%% Coauthor \\
	%% Affiliation \\
	%% Address \\
	%% \texttt{email} \\
}
\aistatstitle{Fisher Auto-Encoders}

\aistatsauthor{ Khalil Elkhalil$^1$ \And Ali Hasan$^{1}$ \And Jie Ding$^2$ \And Sina Farsiu$^1$ \And Vahid Tarokh$^1$}

\aistatsaddress{$^1$Duke University, USA \And  $^2$University of Minnesota, USA} ]

\begin{abstract}
It has been conjectured that the Fisher divergence is more robust to model uncertainty than the conventional Kullback-Leibler (KL) divergence. This motivates the design of a new class of robust generative auto-encoders (AE) referred to as Fisher auto-encoders. Our approach is to design Fisher AEs by minimizing the Fisher divergence between the intractable joint distribution of observed data and latent variables, with that of the postulated/modeled joint distribution.  In contrast to KL-based variational AEs (VAEs), the Fisher AE can exactly quantify the distance between the true and the model-based posterior distributions. Qualitative and quantitative results are provided on both MNIST and celebA datasets demonstrating the competitive performance of Fisher AEs in terms of robustness compared to other AEs such as VAEs and Wasserstein AEs.  

\end{abstract}
\section{Introduction}
In recent years, generative modeling became a very active research area with impressive achievements. The most popular generative schemes are often given by variational auto-encoders (VAEs) \cite{KingmaVAE}, generative adversarial networks (GANs) \cite{GAN} and their variants. VAEs rely on the maximum likelihood principle to learn the underlying data generating distribution by considering a parametric model. Due to the intractability of the parametric model, VAEs employ approximate inference by considering an approximate posterior to get a variational bound on the log-likelihood of the model distribution. Despite its elegance, this approach has the drawback of generating low-quality samples due to the fact that the approximate posterior could be quite different from the true one. On the other hand, GANs have proven to be more impressive when it comes to the visual quality of the generated samples, while the training often involves nontrivial fine-tuning  and is unstable. In addition to difficult training, GANs also suffer from ``mode collapse'' where the generated samples are not diverse enough to capture the diversity and variability in the true data distribution \cite{GAN}.

In this work, we propose a new class of robust auto-encoders that also serve as a generative model.
The main idea is to develop a `score' function~\cite{hyvarinen,scoring} of the observed data and postulated model, so that its minimization problem is equivalent to  minimizing the Fisher divergence~\cite{jie_nips19} between the underlying data generating distribution and the postulated/modeled distribution. By doing this, we are able to leverage the potential advantages of Fisher divergence in terms of computation and robustness. 
In the context of parameter estimation, minimizing the Fisher divergence has led to the Hyv\"arinen score~\cite{hyvarinen}, which   serves as a potential surrogate for the logarithmic score. 
The main advantage of the Hyv\"arinen score over logarithmic sore is its significant computational advantage for estimating probability distributions that are known only up to a multiplicative constant, e.g. those in mixture models and complex time series models~\cite{hyvarinen,scoring,jie_nips19,jieSMC}. %which allows to consider a richer class of distributions. 
Our work will extend the use of Fisher divergence and Hyv\"arinen score in the context of variational auto-encoders.

Similar to the logarithmic score, the Hyv\"arinen score is also intractable to compute due to the intractable integration over the latent variables. One way to mitigate this difficulty is to bound the Hyv\"arinen score and obtain a variational bound to optimize instead. However, unlike the logarithmic score, this strategy seems to be very complicated and a variational bound seems to be out of reach. Alternatively, it turns out that the variational bound in VAEs can be recovered by minimizing the KL divergence between the joint distribution over the data and latent variable and the modeled joint distribution which can be easily calculated as the product of the prior and the decoder distribution \cite{VAE_tutorial}. Following the same principle, we propose to minimize the Fisher divergence between the two joint distributions over the model parameters. This minimization results in a loss function that shares similar properties as regular VAEs but more powerful from an inference point of view. 

It turns out that our developed loss function is the sum of three terms: the first one is the tractable Fisher divergence between the approximate and the model posteriors, the second is similar to the reconstruction loss in VAEs obtained by evaluating the Hyv\"arinen score on the decoder distribution, and the last term can be seen as a stability measure that promotes the invariance property in feature extraction in the encoder. Therefore, the new loss function is different from the regular variational bound in regular VAEs in the following aspects: 1) it considers the minimization of the distance between the approximate and the model posteriors which turns out to be difficult when considering the KL divergence due to the intractable normalization constant in the model posterior, 2) it allows to produce robust features by considering a stability measure of the approximate posterior. Experimental results on MNIST \cite{mnist} and CelebA \cite{celebA} datasets validate these aspects and demonstrate the potential of the proposed Fisher AE as compared to some existing schemes such as VAEs and Wasserstein AEs. Moreover, thanks to the stability measure in the Fisher loss function,  the encoder is proved to have more stable and robust reconstruction when the data is perturbed by noise as compared to other schemes playing a similar role as denoising auto-encoders \cite{DAE}.

\textbf{Related works}. Previous works on learning variational auto-encoders initiated by the work of \cite{KingmaVAE} are fundamentally maximum likelihood methods that learn the underlying data distribution by the proxy of an evidence lower bound (ELBO) on the log-likelihood. The accuracy of such bound is mainly related to the KL divergence between the true and the postulated posteriors. This has been the focus of many works trying to minimize the inference gap resulting from the postulated posterior. For instance, normalizing flows \cite{nf} employs rich posterior approximations using tractable and flexible transformations on initial densities. In the same category, the work in \cite{stein_ae} provides an efficient way of directly sampling from the true posterior using the Stein Variational Gradient Descent (SVGD) method. 
On the other hand, Wasserstein auto-encoders (WAEs) proposed in \cite{wae} follow a different path by looking at the Wasserstein distance between the true and the model distributions. Relying on the Monge-Kantorovich formulation, the Wasserstein distance naturally emerge as an optimization over an encoder-decoder structure with a reasonable geometry over the latent manifold. 

\textbf{Main contributions}. 
First, we develop a new type of AEs that is based on minimizing the Fisher divergence between the underlying data/latent joint distribution and the postulated model joint distribution.
Our derived loss function may be  decomposed as \textit{divergence between posteriors} + \textit{reconstruction loss} + \textit{stability measure}.
Second, our derived method is  conceptually appealing as it is reminiscent of the classical evidence lower bound (ELBO) derived from Kullback-Leibler (KL) divergence.
Third,   we affirmatively address the conjecture made in some earlier work that Fisher divergence can be more robust than KL divergence in modeling complex nonlinear models~\cite{jie_nips19,siwei} in the context of VAEs. Our results indicate that Fisher divergence may serve as a competitive learning machinery for  challenging deep learning tasks.

\textbf{Outline}. 
%The remainder of the paper is organized as follows: 
In Section \ref{background}, we provide a brief overview on VAEs and some theoretical concepts related to the Fisher divergence and the  Hyv\"arinen score. In Section \ref{proposed}, we provide the technical details related to the proposed Fisher auto-encoder. Then, in Section \ref{experiments} we give both qualitative and quantitative results regarding the performance of the proposed Fisher AE. Finally, we provide some concluding remarks in Section \ref{conclusion}.  
\section{Background on VAEs and Fisher divergence}
\label{background}
\subsection{Variational auto-encoders}
By considering a probabilistic model of the data observations $\mathbf{x} \in \mathbb{R}^D$ given by $p_{\theta}(\mathbf{x})$, the goal of variational inference is to optimize the model parameters $\theta$ to match the true unknown data distribution $p_{\star}(\mathbf{x})$ in some sense. One way to match the true data distribution is to minimize the Kullback-Leibler (KL) divergence as follows:
\begin{equation}
\label{KL_true _dist}
\begin{split}
    \theta^{\star} & = \arg \min_{\theta} \mathbb{D}_{\text{KL}} \left[ p_{\star}|| p_{\theta}\right] \\
    & = \arg \min_{\theta} \mathbb{E}_{p_{\star}(\mathbf{x})}  - \log p_{\theta}(\mathbf{x}) \\
    & = \arg \min_{\theta} \mathbb{E}_{p_{\star}(\mathbf{x})} - \log \int p(\mathbf{z}) p_{\theta}(\mathbf{x}|\mathbf{z}) d\mathbf{z},
\end{split}
\end{equation}
where $\mathbf{z} \in \mathbb{R}^d$ are latent variables with prior distribution $p(\mathbf{z})$ and $p_{\theta}(\mathbf{x}|\mathbf{z})$ is a likelihood function corresponding to the decoder modeled by the parameters $\theta$ using a neural network. Unfortunately, the intergration over the latent variables $\mathbf{z}$ in \eqref{KL_true _dist} is usually intractable and an upper bound on the negative marginal log-likelihood is often optimized instead. By introducing an alternative posterior over the latent variables given by $q_{\phi}(\mathbf{z}|\mathbf{x})$ and by direct application of the Jensen's inequality, we have 
\begin{equation}
\label{elbo}
\begin{split}
- \log p_{\theta}(\mathbf{x})  & = -   \log \int \frac{q_{\phi}(\mathbf{z}|\mathbf{x})}{q_{\phi}(\mathbf{z}|\mathbf{x})}p(\mathbf{z}) p_{\theta}(\mathbf{x}|\mathbf{z}) d\mathbf{z} \\
& \leq \mathbb{D}_{\text{KL}} \left[ q_{\phi}(\mathbf{z}|\mathbf{x}) || p(\mathbf{z}) \right] - \mathbb{E}_{q_{\phi}(\mathbf{z}|\mathbf{x})} \log p_{\theta}(\mathbf{x}|\mathbf{z}) \\ 
& = \mathcal{L}_{\text{VAE}} \left(\mathbf{x}; \phi, \theta\right),
\end{split}
\end{equation}
where $q_{\phi}(\mathbf{z}|\mathbf{x})$ is an approximate posterior corresponding to the encoder parameterized by $\phi$. 
The bound in \eqref{elbo} is often called the evidence lower bound (ELBO) (w.r.t the log-likelihood) and it is optimized w.r.t both model parameters $\phi$ and $\theta$:
\begin{equation}
\label{elbo_optim}
    \phi^*, \theta^* = \arg \min_{\phi, \theta} \mathbb{E}_{p_{\star}(\mathbf{x})}  \mathcal{L}_{\text{VAE}} \left(\mathbf{x}; \phi, \theta\right).
\end{equation}
The common practice is to consider a Gaussian model for the posterior $q_{\phi}(\mathbf{z}|\mathbf{x})$, i.e., $q_{\phi}(\mathbf{z}|\mathbf{x}) = \mathcal{N}\left(\mathbf{z}| \mu(\mathbf{x}), \sigma(\mathbf{x})^2 \right)$ where $\mu(\mathbf{x})$ and $\sigma(\mathbf{x})^2$ are the output of a neural network taking as input the data sample $\mathbf{x}$ and parameterized by $\phi$. This allows to reparametrize $\mathbf{z}$ as $\mathbf{z} = \mu(\mathbf{x}) + \sigma(\mathbf{x}) \odot \epsilon$, where $\odot$ denotes the point-wise multiplication and $\epsilon \sim \mathcal{N}\left(0, \mathbf{I} \right)$ which permits to efficiently solve \eqref{elbo_optim} using stochastic gradient variational Bayes (SGVB) as in \cite{KingmaVAE}.

\subsection{Fisher divergence and the Hyv\"arinen score}
A standard procedure in data fitting and density estimation is to select from a parameter space $\Theta$, the probability distribution $p_{\theta}$, $\theta \in \Theta$ that minimizes a certain divergence $\mathbb{D}\left[. || .\right]$ with respect to the unknown true data distribution $p_{\star}$. For a certain class of divergences, expanding the divergence w.r.t the true probability distribution yields: $\mathbb{D} \left[ p_{\star} || p_{\theta}\right]  = c_{\star} + \mathbb{E}_{p_{\star}(\mathbf{x})} s\left[ p_{\theta}\left(\mathbf{x}\right)\right]$, where $c_{\star}$ is a constant that depends only on the data and $s\left[ .\right]: \mathbb{R}^+ \to \mathbb{R}$ is a score function associated to $\mathbb{D}[. || .]$. Clearly, the smaller the score $s\left[ p_{\theta}\left(\mathbf{x}\right)\right]$, the better the data point $\mathbf{x} \sim p_{\star} $ fits the model $p_{\theta}$. In practice, given a set of observations $\{\mathbf{x}_i\}_{i=1, \cdots, N} \sim_{i.i.d} p_{\star}$, one would minimize the sample average $N^{-1} \sum_{i=1}^N s\left[ p_{\theta}\left(\mathbf{x}_i\right)\right]$ over $\theta \in \Theta$.
The most popular example of these scoring functions \cite{scoring} is the logarithmic score given by $- \log p_{\theta}\left(\mathbf{x}\right)$ which is obtained by minimizing the Kullback-Leibler (KL) divergence, i.e. $\mathbb{D} = \mathbb{D}_{\text{KL}}$. In this case, the procedure of minimizing the score function is widely known as maximum likelihood (ML) estimation and has been extensively applied in statistics and machine learning. Popular instances of ML estimation include logistic regression when minimizing the cross-entropy loss w.r.t a Bernoulli model of the data and regression when minimizing the squared loss in the presence of a Gaussian model of the data \cite{bishop}. In the context of variational inference, the logarithmic score is fundamental in the construction of variational autoencoders \cite{KingmaVAE} as we showed in the previous section. 

Recently, the Hyv\"arinen score \cite{hyvarinen, jie_nips19, pmlr-v48-liub16, siwei} that we denote by $s_{\nabla}[.]$ has been proposed as an alternative to the logarithmic score. 
It turns out that the Hyv\"arinen score can be obtained by minimizing the Fisher divergence defined as
\begin{equation}
    \label{fisher_div}
    \mathbb{D}_{\nabla} \left[ p_{\star} || p_{\theta}\right] = \mathbb{E}_{p_{\star}(\mathbf{x})} \frac{1}{2} \left \| \nabla_{\mathbf{x}} \log p_{\star}(\mathbf{x}) - \nabla_{\mathbf{x}} \log p_{\theta} (\mathbf{x})\right \|^2,
\end{equation}
where $\nabla_{\mathbf{x}}$ denotes the gradient w.r.t $\mathbf{x}$. 
Assuming the same regularity conditions as in Proposition 1 \cite{jie_nips19}, we have 
\begin{equation}
\label{fisher_div_score}
        \mathbb{D}_{\nabla} \left[ p_{\star} || p_{\theta}\right] = \mathbb{E}_{p_{\star}(\mathbf{x})} \frac{1}{2} \left \| \nabla_{\mathbf{x}} \log p_{\star}(\mathbf{x}) \right \|^2 + s_{\nabla} \left[ p_{\theta} (\mathbf{x})\right],
\end{equation}
with
\begin{equation}
    \label{H_score}
    s_{\nabla} \left[ p(\mathbf{x})\right] = \frac{1}{2} \left \|\nabla_{\mathbf{x}} \log p(\mathbf{x}) \right \|^2 + \Delta_{\mathbf{x}} \log p(\mathbf{x}),
\end{equation}
for some probability density function $p(\mathbf{x})$ and $\Delta_{\mathbf{x}} = \sum_{j=1}^D \frac{\partial^2}{\partial x_j^2} f(\mathbf{x})$ denotes the Laplacian of some function $f$ w.r.t $\mathbf{x}$. The potential of both the Fisher divergence and the Hyv\"arinen score is their ability to deal with probability distributions that are known up to some multiplicative constant. This interesting property allows to consider larger class of unormalized distributions and therefore better fits the data. In the next section, we provide a detailed description of how we can extend the use of Fisher divergence and Hyv\"arinen score in the context of variational auto-encoders. 
\section{Proposed Fisher Auto-Encoder}
\label{proposed}
Recall from \eqref{elbo} that instead of minimizing the logarithmic score $-\log p_{\theta} (\mathbf{x})$, we instead upper bound the score and minimize $\mathcal{L}_{\text{VAE}} \left(\mathbf{x}; \phi, \theta\right)$. Similarly, one would look for an upper bound to the Hyv\"arinen score $s_{\nabla}\left[ p_{\theta}(\mathbf{x})\right]$ and minimize it w.r.t model parameters $\phi$ and $\theta$. However, this is quite non-trivial as opposed to the logarithmic score in \eqref{elbo}. Fortunately, the upper bound in \eqref{elbo} can be recovered by minimizing the KL divergence between the following two joint distributions: $q_{\star, \phi}(\mathbf{x}, \mathbf{z}) = p_{\star}(\mathbf{x}) q_{\phi}(\mathbf{z} | \mathbf{x})$ and $p_{\eta, \theta}(\mathbf{x}, \mathbf{z}) = p_{\eta}(\mathbf{z}) p_{\theta}(\mathbf{x} | \mathbf{z})$ where $q_{\phi}(\mathbf{z} | \mathbf{x})$, $p_{\eta}(\mathbf{z})$ and $p_{\theta}(\mathbf{x} | \mathbf{z})$ are respectively the variational posterior, the prior and the decoder with parameters $\phi$, $\eta$ and $\theta$. 
\begin{equation}
\begin{split}
        & \phi^{\star}_{\text{VAE}}, \eta^{\star}_{\text{VAE}}, \theta^{\star}_{\text{VAE}} \\ & = \arg \min_{\phi, \eta, \theta} \mathbb{D}_{\text{KL}} \left[q_{\star, \phi}(\mathbf{x}, \mathbf{z}) || p_{\eta, \theta}(\mathbf{x}, \mathbf{z}) \right] \\ 
    & = \arg \min_{\phi, \eta, \theta} \mathbb{E}_{p_{\star}(\mathbf{x})} \mathbb{E}_{q_{\phi}(\mathbf{z}|\mathbf{x})} \left[ \log p_{\star}(\mathbf{x}) + \log \frac{q_{\phi}(\mathbf{z}|\mathbf{x})}{p_{\eta}(\mathbf{z}) p_\theta (\mathbf{x}|\mathbf{z})} \right] \\ 
    & = \arg \min_{\phi, \eta, \theta} \mathbb{E}_{p_{\star}(\mathbf{x})} \{ \mathbb{D}_{\text{KL}} \left[q_{\phi}(\mathbf{z}|\mathbf{x}) || p_{\eta}(\mathbf{z}) \right]  \\  & \quad \quad \quad \quad \quad \quad \quad \quad - \mathbb{E}_{q_{\phi}(\mathbf{z}|\mathbf{x}) } \log p_{\theta}(\mathbf{x}|\mathbf{z}) \}\\ 
    & = \arg \min_{\phi, \eta, \theta} \mathbb{E}_{p_{\star}(\mathbf{x})} \mathcal{L}_{\text{VAE}} \left(\mathbf{x}; \phi, \eta, \theta\right).
\end{split}
\end{equation}
Following the same line of thought, we propose to minimize the Fisher divergence between $q_{\star, \phi}(\mathbf{x}, \mathbf{z})$ and $p_{\eta, \theta}(\mathbf{x}, \mathbf{z})$ as follows:
\begin{equation}
\label{fisher_div_joint}
    \begin{split}
        & \phi^{\star}, \eta^{\star}, \theta^{\star}  \\ & = \arg \min_{\phi, \eta, \theta} \mathbb{D}_{\nabla}    \left[q_{\star, \phi}(\mathbf{x},\mathbf{z}) || p_{\eta, \theta}(\mathbf{x}, \mathbf{z}) \right] \\ 
 & = \arg \min_{\phi, \eta, \theta}
 \\ &  \mathbb{E}_{q_{\star, \phi} (\mathbf{x}, \mathbf{z})}  \frac{1}{2} \left \| \nabla_{\mathbf{x}, \mathbf{z}} \log q_{\star, \phi}(\mathbf{x},\mathbf{z}) - \nabla_{\mathbf{x}, \mathbf{z}} \log p_{\eta, \theta}(\mathbf{x},\mathbf{z})\right\|^2,
    \end{split}
\end{equation}
where $\nabla_{\mathbf{x}, \mathbf{z}}$ denotes the gradient w.r.t the augmented variable $\{\mathbf{x}, \mathbf{z}\}$. The following theorem provides a simplified expression of the Fisher AE loss by expanding and simplifying the Fisher divergence in \eqref{fisher_div_joint}. 
\begin{theorem}
\label{theorem_fisher_loss}
The minimization in \eqref{fisher_div_joint} is equivalent to the following minimization problem:
\begin{equation}
    \label{fisher_minimization}
    \begin{split}
    \phi^{\star}, \eta^{\star}, \theta^{\star} & = \arg \min_{\phi, \eta, \theta} \mathbb{D}_{\nabla}    \left[q_{\star, \phi}(\mathbf{x},\mathbf{z}) || p_{\eta, \theta}(\mathbf{x}, \mathbf{z}) \right] \\ 
 & =   \arg \min_{\phi, \eta, \theta} \mathbb{E}_{p_{\star}(\mathbf{x})} \mathcal{L}_{\text{F-AE}} \left(\mathbf{x}; \phi, \eta, \theta\right),     
    \end{split}
\end{equation}
where 
\begin{equation}
\label{fisher_loss}
\begin{split}
        & \mathcal{L}_{\text{F-AE}} \left(\mathbf{x}; \phi, \eta, \theta\right) \\ 
        & =  \underbrace{\mathbb{D}_{\nabla} \left[
 q_{\phi}(\mathbf{z}|\mathbf{x}) || p_{\eta, \theta}(\mathbf{z}|\mathbf{x})\right]}_{\textcircled{1}} \\ & +  \mathbb{E}_{q_{\phi}(\mathbf{z}|\mathbf{x})} \biggl[ \underbrace{s_{\nabla}\left[ p_{\theta}(\mathbf{x}|\mathbf{z}) \right]}_{\textcircled{2}} + \underbrace{\frac{1}{2} \left \|\nabla_{\mathbf{x}} \log q_{\phi}(\mathbf{z}|\mathbf{x}) \right \|^2}_{\textcircled{3}}\biggr] .
 \end{split}
\end{equation}
\end{theorem}
\begin{proof}
A proof can be found in the supplementary material. 
\end{proof}
The Fisher AE loss denoted by $\mathcal{L}_{\text{F-AE}}\left(\mathbf{x}; \phi, \eta, \theta\right)$ in \eqref{fisher_loss} is the sum of the following three terms: \textcircled{1} the Fisher divergence between the two posteriors $q_{\phi}(\mathbf{z}|\mathbf{x})$ and $p_{\eta, \theta}(\mathbf{z}|\mathbf{x})$. In traditional VAEs, the KL divergence between these two posteriors is generally intractable since $p_{\eta, \theta}(\mathbf{z}|\mathbf{x}) = \frac{p_{\eta}(\mathbf{z}) p_{\theta}(\mathbf{x} | \mathbf{z})}{p_{\eta, \theta}(\mathbf{x})}$ and $p_{\eta, \theta}(\mathbf{x})$ is hard to compute because $p_{\eta, \theta}(\mathbf{x}) = \int p_{\eta}(\mathbf{z}) p_{\theta}(\mathbf{x} | \mathbf{z}) d\mathbf{z}$. Interestingly, with the Fisher divergence this limitation is alleviated since $ p_{\eta, \theta}(\mathbf{z}|\mathbf{x}) \propto p_{\eta}(\mathbf{z}) p_{\theta}(\mathbf{x} | \mathbf{z}) $ and we only need $\nabla_{\mathbf{z}} \log p_{\eta, \theta}(\mathbf{z}|\mathbf{x}) = \nabla_{\mathbf{z}} \log p_{\eta} (\mathbf{z}) + \nabla_{\mathbf{z}} \log p_{\theta}(\mathbf{x} | \mathbf{z}) $ for computation. The second term given by \textcircled{2} is the Hyv\"arinen score of $p_{\theta}(\mathbf{x} | \mathbf{z})$  which is nothing but a reconstruction loss similar to $- \log p_{\theta} (\mathbf{x} | \mathbf{z})$ in regular VAEs. When $p_{\theta}(\mathbf{x} | \mathbf{z}) \propto e^{ - \frac{1}{2}\left \| \mathbf{x} - f_{\theta}(\mathbf{z}) \right \|^2}$, the reconstruction loss is given by the squared loss\footnote{We omit the constant term coming from the Laplacian $\Delta_{\mathbf{x}} \log p_{\theta}(\mathbf{x} | \mathbf{z})$ since it is irrelevant to the minimization problem in \eqref{fisher_minimization}.}: $\frac{1}{2}\left \| \mathbf{x} - f_{\theta}(\mathbf{z}) \right \|^2$ which is the same as in regular VAEs under the same model, $f_{\theta}(.): \mathbb{R}^d \to \mathbb{R}^D$ is the decoder parametrized by $\theta$. The last term \textcircled{3} is a stability term that permits to produce robust features in the sense that the posterior distribution is robust against small perturbations in the input data. This is similar to contractive auto-encoders which promote the invariance property in feature extraction \cite{contractive_ae}. 
\begin{remark} 
\label{equal_posteriors}
When $q_{\phi}(\mathbf{z}|\mathbf{x}) = p_{\eta, \theta}(\mathbf{z}|\mathbf{x})$, the Fisher AE loss becomes exactly the Hyv\"arinen score of the model  distribution $p_{\eta, \theta}(\mathbf{x})$, i.e. $\mathcal{L}_{\text{F-AE}} \left(\mathbf{x}; \phi, \eta,  \theta\right) = s_{\nabla}\left[ p_{\eta, \theta}(\mathbf{x})\right]$. This is similar to traditional VAEs since we also have $\mathcal{L}_{\text{VAE}} \left(\mathbf{x}; \phi, \eta, \theta\right) = - \log p_{\eta, \theta}(\mathbf{x})$ in this case. 
\end{remark}
\begin{proof}
When $q_{\phi}(\mathbf{z}|\mathbf{x}) = p_{\eta, \theta}(\mathbf{z}|\mathbf{x})$,  $\mathbb{D}_{\nabla}    \left[q_{\star, \phi}(\mathbf{x},\mathbf{z}) || p_{\eta, \theta}(\mathbf{x}, \mathbf{z}) \right] = \mathbb{D}_{\nabla} \left[ p_{\star}(\mathbf{x}) || p_{\eta, \theta} (\mathbf{x})\right]$. The proof is concluded by relying on \eqref{fisher_div_score}.
\end{proof}
Given a data point $\mathbf{x}$, the Fisher AE loss can be estimated using Monte Carlo with $L$ samples from $q_{\phi}(\mathbf{z}|\mathbf{x})$ as follows:
\begin{equation}
\label{MC_FisherLoss}
\begin{split}
 & \mathcal{L}_{\text{F-AE}} \left(\mathbf{x}; \phi, \eta, \theta\right) \\  & \simeq   \mathcal{L}^{(L)}_{\text{F-AE}} \left(\mathbf{x}; \phi, \eta, \theta\right) \\
 & = 
 \frac{1}{2L} \sum_{l=1}^L \Biggl[\|\nabla_{\mathbf{z}} \log q_{\phi}(\mathbf{z}^{(l)}|\mathbf{x}) - \nabla_{\mathbf{z}} \log p_{\eta}(\mathbf{z}^{(l)}) \\ & \quad \quad \quad \quad \quad \quad - \nabla_{\mathbf{z}} \log p_{\theta}(\mathbf{x}| \mathbf{z}^{(l)}) \|^2  \\ 
 & \quad \quad \quad \quad + \left \| \mathbf{x} - f_{\theta}(\mathbf{z}^{(l)}) \right \|^2 + \left \| \nabla_{\mathbf{x}} \log q_{\phi}(\mathbf{z}^{(l)}|\mathbf{x}) \right \|^2 \Biggr],
\end{split}
\end{equation}
where $\mathbf{z}^{(l)} = \mu(\mathbf{x}) + \sigma(\mathbf{x}) \odot \epsilon^{(l)}$, $\epsilon^{(l)} \sim \mathcal{N} \left(0, \mathbf{I} \right)$. Moreover, $\nabla_{\mathbf{z}} \log q_{\phi}(\mathbf{z}^{(l)}|\mathbf{x}) = -\frac{\epsilon^{(l)}}{\sigma(\mathbf{x})} $ and both $\nabla_{\mathbf{z}} \log p_{\theta}(\mathbf{x}| \mathbf{z}^{(l)})$ and $\nabla_{\mathbf{x}} \log q_{\phi}(\mathbf{z}^{(l)}|\mathbf{x})$ can be computed using automatic differentiation tools like Autograd in PyTorch. To solve the minimization in \eqref{fisher_minimization}, we use stochastic gradient descent (SGD) with minibatch data of size $N$ as in \cite{KingmaVAE}. Details of the optimization are given by Algorithm \ref{SGD}.  
\begin{algorithm}[h!]
	\caption{Training the Fisher AE with SGD} 
	 \label{SGD}
	\begin{algorithmic}[1]
	\STATE \textbf{Initialize} $\phi$, $\eta$ and $\theta$ \\ 
    \STATE \textbf{Repeat}: \\ 
    \STATE \quad Randomly sample a minibatch of training data $\{\mathbf{x}_i\}_{i=1}^N$ 
    \STATE \quad Compute gradient
     $\nabla_{\phi, \eta, \theta} \frac{1}{N} \sum_{i=1}^N \mathcal{L}^{(L)}_{\text{F-AE}} \left(\mathbf{x}_i; \phi, \eta, \theta\right)$ \\
    \STATE \quad Update $\phi$, $\eta$ and $\theta$ with Adam \cite{adam}
    \STATE \textbf{Until} convergence
    \STATE \textbf{Output:} $\phi_{*}$, $\eta_*$ and $\theta_*$
	\end{algorithmic} 
\end{algorithm}

\subsection{Fisher AE with exponential family priors}
As discussed earlier, employing the Fisher divergence has the advantage of dealing with probability distributions that are known up to some multiplicative constant. This powerful property allows to consider a rich family of distributions to model the prior $p(\mathbf{z})$. In this paper, we consider the use of exponential family whose general form is given by:
\begin{equation}
    \label{exp_family_general}
    p_{\eta}(\mathbf{z})  \propto \exp \left( \eta^{\top}T(\mathbf{z}) + h(\mathbf{z})\right), 
\end{equation}
where $\eta$ denotes the natural parameters, $h(\mathbf{z})$ is the carrier measure and
$T(\mathbf{z})$ is referred to as a sufficient statistic \cite{exp_family}. Popular examples of the exponential family include the Bernoulli, Poisson and Gaussian distributions to name a few \cite{exp_family}.
Note that the form given by the right hand side of \eqref{exp_family_general} is not a valid PDF since it does not sum to 1, but it is sufficient to compute the gradient of the log-density w.r.t $\mathbf{z}$ which is given by $\nabla_{\mathbf{z}} \log p_{\eta}(\mathbf{z})  = \nabla_{\mathbf{z}} \left( \eta^{\top}T(\mathbf{z}) + h(\mathbf{z})\right)$. Therefore, the term \textcircled{1} in \eqref{fisher_loss} can be written as:
\begin{align*}
&  \mathbb{D}_{\nabla} \left[
 q_{\phi}(\mathbf{z}|\mathbf{x}) || p_{\theta}(\mathbf{z}|\mathbf{x})\right] \\
 & = \frac{1}{2} \int q_{\phi}(\mathbf{z}|\mathbf{x}) \|\nabla_{\mathbf{z}} \log q_{\phi}(\mathbf{z}|\mathbf{x}) - \nabla_{\mathbf{z}} \left( \eta^{\top}T(\mathbf{z}) + h(\mathbf{z})\right)  \\ & \quad \quad \quad \quad \quad \quad \quad - \nabla_{\mathbf{z}} \log p_{\theta}(\mathbf{x}| \mathbf{z}) \|^2 d \mathbf{z}. 
\end{align*}
which can be approximated using samples $\mathbf{z}^{(l)} \sim q_{\phi}(\mathbf{z}|\mathbf{x}) $, $l=1, \cdots, L$ as follows: 
\begin{align*}
 & \mathbb{D}_{\nabla} \left[
 q_{\phi}(\mathbf{z}|\mathbf{x}) || p_{\theta}(\mathbf{z}|\mathbf{x})\right] \\ 
 & \simeq  \frac{1}{2L} \sum_{l=1}^L    \|\nabla_{\mathbf{z}} \log q_{\phi}(\mathbf{z}^{(l)}|\mathbf{x}) - \nabla_{\mathbf{z}} \left( \eta^{\top}T(\mathbf{z}^{(l)}) + h(\mathbf{z}^{(l)})\right)  \\ & \quad \quad \quad \quad \quad \quad - \nabla_{\mathbf{z}} \log p_{\theta}(\mathbf{x}| \mathbf{z}^{(l)}) \|^2. 
\end{align*}
A popular class of distributions that belongs to the exponential family is given by the factorable polynomial exponential family (FPE) \cite{polynomial_family} in which $p_{\eta}(\mathbf{z})$ is given by 
\begin{equation}
    \label{FPE}
    p_{\eta}(\mathbf{z}) = p_{\eta}(z_1, \cdots, z_d) \propto \exp \left( \sum_{j=1}^d \sum_{k=1}^K \eta_{jk} z_j^k \right),
\end{equation}
where $K$ denotes the order of FPE family and $\{\eta_{jk}\}_{1 \leq j \leq d, 1 \leq k \leq K}$ is a set of parameters. The natural parameters, the sufficient statistic and the carrier measure in this case are given by:
\begin{align*}
    \eta & = \left[\eta_{11}, \eta_{12}, \cdots, \eta_{1K}, \cdots, \eta_{d1}, \eta_{d2}, \cdots, \eta_{dK}  \right]^{\top} \\ 
   T(\mathbf{z}) & = \left[z_1, z_1^2, \cdots, z_1^K, \cdots, z_d, z_d^2, \cdots, z_d^K \right]^{\top} \\
    h(\mathbf{z}) & = 0.
\end{align*}
With the model in \eqref{FPE}, the gradient of $\log p_{\eta}(\mathbf{z})$ w.r.t $\mathbf{z}$ can be easily derived as 
\begin{align*}
\frac{\partial}{\partial z_j} \log p_{\eta}(\mathbf{z}) = \sum_{k=1}^K k \eta_{jk} z_j^{k-1}, \quad j=1, \cdots, d.
\end{align*}
\section{Experiments}
\label{experiments}
In this section, we provide both qualitative and quantitative results that demonstrate the ability of our proposed Fisher AE model to produce high quality samples on real-world image datasets such as MNIST and CelebA. We compare results with both regular VAEs \cite{KingmaVAE} and Wasserstein Auto-Encoders with GAN penalty (WAE-GAN) \cite{wae}. In the supplementary material, we provide full details for the encoder/decoder architectures used by the different schemes for both MNIST and celebA datasets.  
\subsection*{Setup}
For optimization, we use Adam \cite{adam} with a learning rate $\texttt{lr}=2.10^{-4}$, $\beta_1=0.5$, $\beta_2=0.999$, a mini-batch size of $128$ and trained various models for 100 epochs. For all experiments, we pick $d=8$ for MNIST and $d=64$ for celebA and use Gaussian and Bernoulli decoders for Fisher AE and regular VAE respectively. As proposed earlier, we use exponential family priors for the Fisher AE as in \eqref{FPE} and noticed that $K=5$ seems to work better in all experiments whereas Gaussian priors are used for VAE and WAE-GAN. We use Gaussian posteriors for both Fisher AE and VAE such that $q_{\phi}(\mathbf{z}| \mathbf{x}) = \mathcal{N}\left(\mathbf{z}; \mu_{\phi}(\mathbf{x}), \sigma_{\phi}(\mathbf{x})^2 \right)$ where $ \mu_{\phi}(.)$ and $\sigma_{\phi}(.)$ are determined by the encoder architecture for which details are postponed to the supplementary material.

\subsection*{Sampling with SVGD}
To sample from the exponential family prior after training, we use Stein Variational Gradient Descent (SVGD) \cite{svgd_general} . Let $M$ be the number of samples that we would like to sample from $p_{\eta_*}(\mathbf{z})$ denoted by $\left\{\mathbf{z}^*_i \right\}_{i=1}^M$. We start with $\left \{ \mathbf{z}_i\right \}_{i=1}^M \sim_{i.i.d} \mathcal{N}(0, \mathbf{I})$ and we keep evolving these samples with a step-size $10^{-3}$ for $15,000$ iterations. These parameters (step-size and number of iterations) seem to work reasonably well across all experiments. 
\subsection*{MNIST} 
\begin{figure}
      \begin{center}
    \includegraphics[scale=0.45]{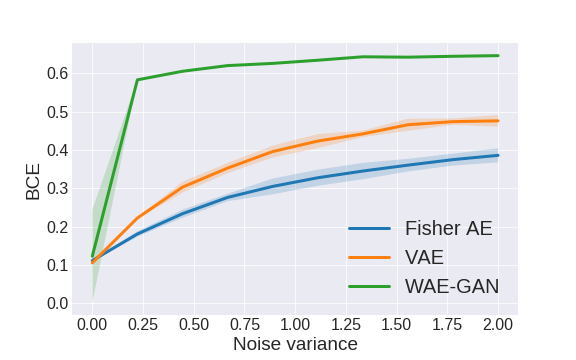}
  \end{center}
    \caption{BCE vs. noise variance $\sigma^2$ for MNIST.}
    \label{fig:BCE_noise_mnist}
\end{figure}
Figure \ref{fig:mnist_results_main} exhibits a comparison between the three auto-encoders in terms of robustness, test reconstruction, and random sampling. In order to compare the robustness, we plot the reconstructed samples of the different schemes when the test data is corrupted by an isotropic Gaussian noise with a covariance matrix $0.2 \times \mathbf{I}_D$. 
% In order to compare the robustness,
% we train the different auto-encoders with noisy data (we added zero-mean isotropic Gaussian noise with covariance $0.2 \times \mathbf{I}_D$ to the training data) and examine the reconstruction performance of the different auto-encoders. 
The results of this experiment are given by the first row of Figure \ref{fig:mnist_results_main}.
Clearly, WAE-GAN completely fail to reconstruct the test data and Fisher AE seems to be more robust to noise. This result is confirmed quantitatively in Figure \ref{fig:BCE_noise_mnist} where we plot the normalized binary cross-entropy (BCE)  w.r.t the noise variance added to the test data, i.e. we feed the different trained models with $\texttt{data} = \texttt{test data} + \mathcal{N}(0, \sigma^2 \mathbf{I}_D)$ and compute the BCE reconstruction loss w.r.t the true test data.  In the second and third rows of Figure \ref{fig:mnist_results_main}, we show both the reconstruction and generative performance of the different auto-encoders. For both test reconstruction and random sampling, the proposed Fisher AE exhibits a comparable performance to WAE-GAN which achieves the best generative performance thanks to the GAN penalty in the loss function \cite{wae}. 

We further examine the robustness of the different models w.r.t latent representation when the data is corrupted by additive Gaussian noise. In Figure \ref{fig:robust_cluster_fisher} using non-linear dimensionality reduction techniques such as t-SNE, we visualize the 2D latent structure of the different models. As shown in Figure \ref{fig:robust_cluster_fisher}, even with corrupted data, the latent structure of the Fisher AE is still preserved and the clusters associated to different classes are relatively distinguishable. This is not the case for VAE and WAE-GAN where the clusters in the latent space are somewhat mixed up when the data is perturbed by noise. This behavior is quantitatively confirmed in Figure \ref{fig:mutual_information} where test data is perturbed with Gaussian noise with variance $\sigma^2$, then encoded with each model encoder and projected with t-SNE and finally clustered using k-means. The qualtiy of clustering is measured using the normalized mutual information: $\text{NMI}= \frac{2 \mathcal{I}\left(\mathbb{\Omega}; \mathbb{C} \right)}{\mathcal{H}(\mathbb{\Omega}) + \mathcal{H}(\mathbb{C})}$, where $\mathbb{\Omega}$ is the model clusters for a given noise level, $\mathbb{C}$ is the true class labels, $\mathcal{I}(.;.)$ denotes the mutual information and $\mathcal{H}(.)$ denotes the entropy. Clearly, Fisher AE exhibits a better behavior in terms of clustering robustness where the decay in performance is nearly linear whereas for both VAE and WAE-GAN, the performance decays faster. 
\begin{figure}
    \centering
    % \begin{subfigure}[b]{0.3\textwidth}
    %     \includegraphics[width=\textwidth]{KL_VAE.png}
    %     \caption{Regular VAE}
    %     \label{fig:vae_samples}
    % \end{subfigure}
    % ~ %add desired spacing between images, e. g. ~, \quad, \qquad, \hfill etc. 
    %   %(or a blank line to force the subfigure onto a new line)
    % \begin{subfigure}[b]{0.3\textwidth}
    %     \includegraphics[width=\textwidth]{Fisher_VAE.png}
    %     \caption{Fisher VAE}
    %     \label{fig:fisher_samples}
    % \end{subfigure}
    % ~ %add desired spacing between images, e. g. ~, \quad, \qquad, \hfill etc. 
    % %(or a blank line to force the subfigure onto a new line)
    % \begin{subfigure}[b]{0.3\textwidth}
    %     \includegraphics[width=\textwidth]{Test_data.png}
    %     \caption{Test data}
    %     \label{fig:tes_samples}
    % \end{subfigure}
     \includegraphics[scale=0.36]{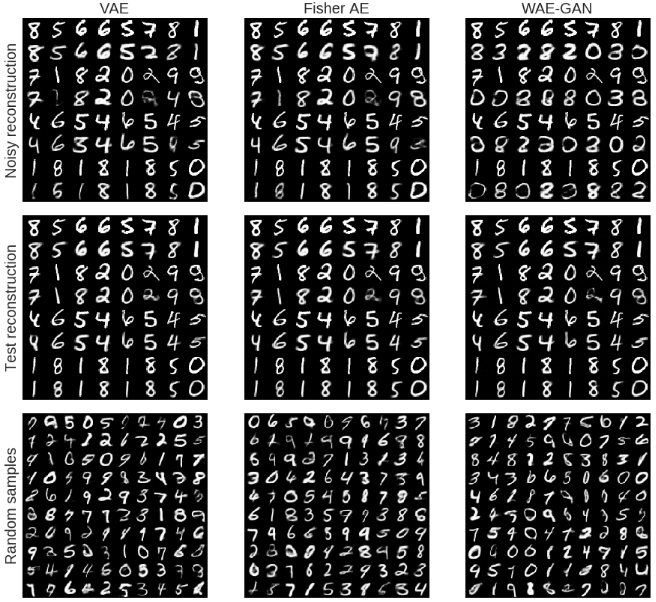}
    \caption{Performance of the Fisher AE trained on MNIST dataset in comparison with VAE and WAE-GAN. True test data are given by the odd rows in both reconstruction tasks (rows 1 and 2). }
    \label{fig:mnist_results_main}
\end{figure}
\begin{figure}[h!]
    \centering
     \includegraphics[scale=0.56]{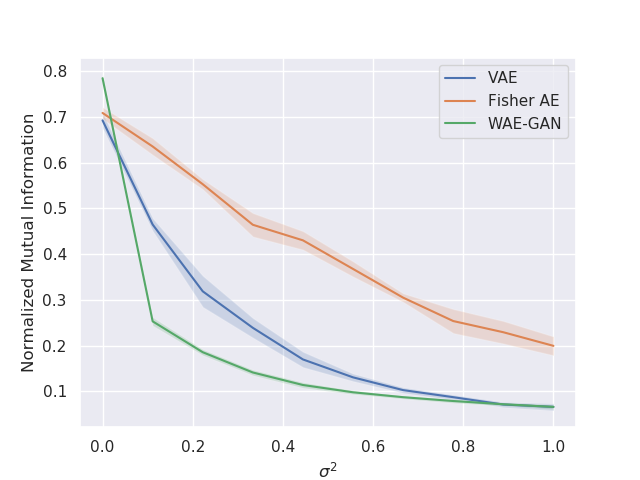}
    \caption{Robustness of latent space clustering in terms of the normalized mutual information on MNIST test set.}
    \label{fig:mutual_information}
\end{figure}
\begin{figure*}
    \centering
     \includegraphics[scale=0.86]{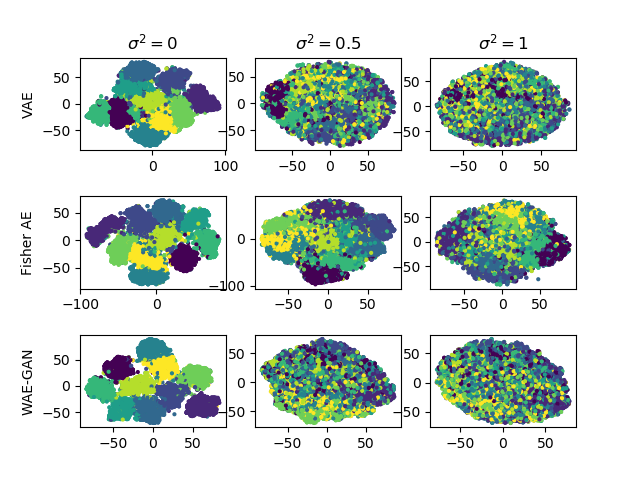}
    \caption{Visualisation of models latent representation on MNIST test set using t-SNE for different noise levels.}
    \label{fig:robust_cluster_fisher}
\end{figure*}
\subsection*{CelebA}
\begin{figure}[h]
    \includegraphics[scale=0.42]{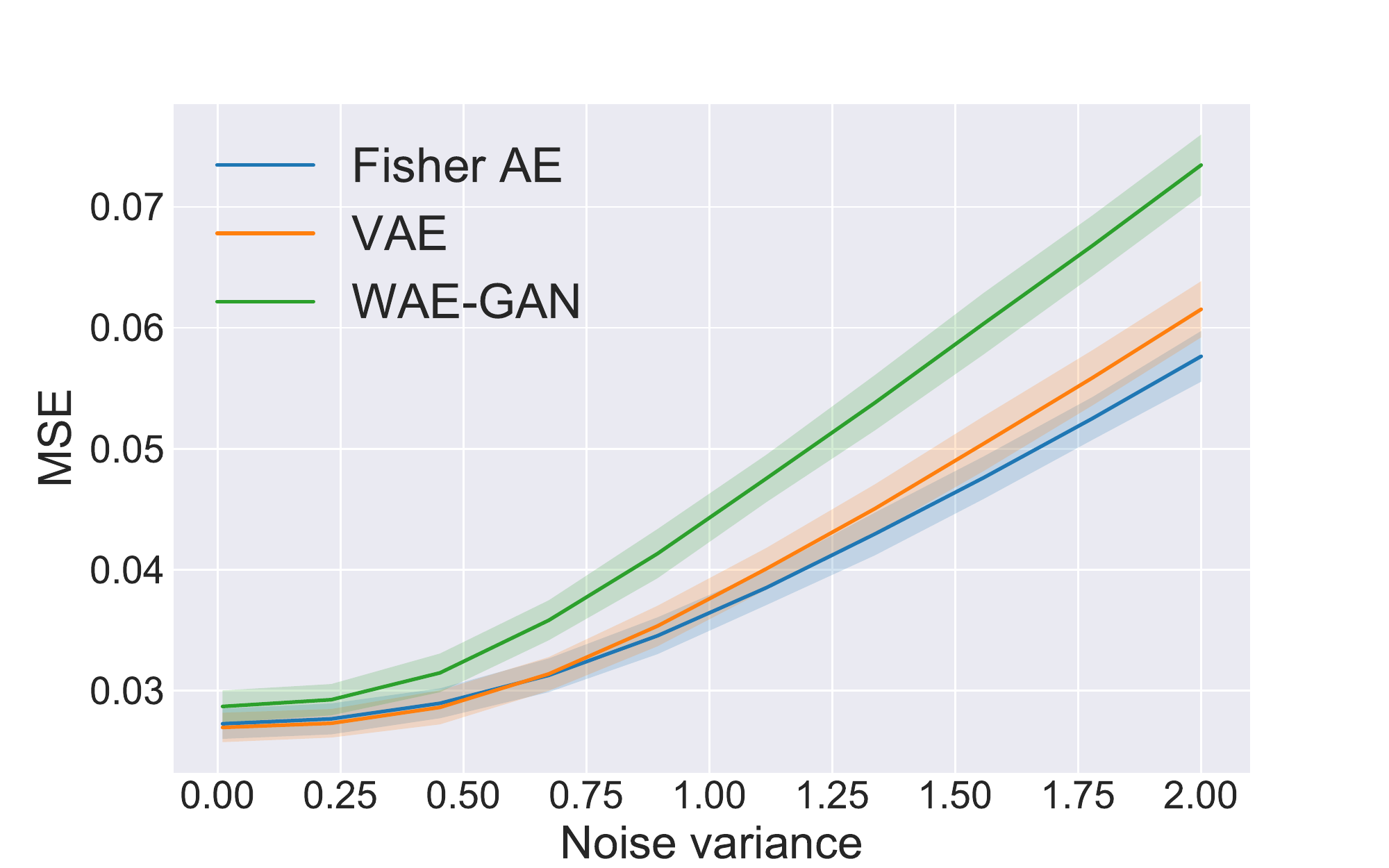}
    \caption{MSE vs. noise variance $\sigma^2$ for celebA. Errors are computed from variances in batches in the test set.}
    \label{fig:MSE_noise_celeba}
\end{figure}
\begin{table}
\caption{FID scores of the different generative models trained on CelebA (smaller is better).}
\begin{tabular}{ll}
\hline
\textbf{Algorithm}                                                                & \textbf{FID score}                                                     \\ \hline
\begin{tabular}[c]{@{}l@{}}VAE\\ Fisher AE (Gaussian prior)\\
\textbf{Fisher AE (Exp. prior)} \\ 
WAE-GAN\end{tabular} & \begin{tabular}[c]{@{}l@{}}89.1 $\pm$ 1.1 \\ 89.1 $\pm$ 0.9 \\ \textbf{84.7 $\pm$ 0.8} \\  75.2 $\pm$ 1.0 \end{tabular} \\ \hline
\end{tabular}
\label{FID_scores}
\end{table}
% \begin{figure}
%       %\begin{center}
%     \includegraphics[scale=0.34]{mse_plot_err.pdf}
%   %\end{center}
%   \vspace{-5pt}
%     \caption{MSE vs. noise variance $\sigma^2$ for celebA.}
%     \label{fig:MSE_noise_celeba}
% \end{figure}

% \begin{wrapfigure}{r}{0.45\textwidth}
% \vspace{-50pt}
% \hspace{-10pt}
%       %\begin{center}
%     \includegraphics[scale=0.34]{mse_plot_err.pdf}
%   %\end{center}
%   \vspace{-5pt}
%     \caption{MSE vs. noise variance $\sigma^2$ for celebA.}
%     \label{fig:MSE_noise_celeba}
% \end{wrapfigure}
% \begin{wrapfigure}{r}{0.45\textwidth}
% \vspace{-50pt}
% \hspace{-10pt}
%       %\begin{center}
%     \includegraphics[scale=0.34]{mse_plot_err.pdf}
%   %\end{center}
%   \vspace{-5pt}
%     \caption{MSE vs. noise variance $\sigma^2$ for celebA.}
%     \label{fig:MSE_noise_celeba}
% \end{wrapfigure}
For the CelebA dataset, it is clear from the first row (the noisy reconstructions) of Figure \ref{fig:agg_celeba_2} that the proposed Fisher AE is more robust than both VAE and WAE-GAN when the test data is corrupted with an isotropic Gaussian noise with covariance matrix $2 \mathbf{I}_D$. We further validate this property with different noise levels as depicted in Figure \ref{fig:MSE_noise_celeba} where the Fisher AE outperforms VAE and WAE-GAN in the reconstruction MSE. 
Moreover, as shown in Figure \ref{fig:agg_celeba_2}, the Fisher AE generates better samples than VAE and has  comparable quality to WAE. 
% However, it is worth mentioning that training Fisher AE is unstable due to the additional set of parameters modeling the exponential family priors. 
The visual quality of the samples is confirmed by the quantitative results summarized in Table \ref{FID_scores} where the proposed Fisher AE with exponential family priors outperforms VAE in terms of the Fr\'echet Inception Distance (FID) and has  relatively worse performance than WAE. Furthermore, sampling using the exponential prior provides additional challenges due to the difficulty of convergence of the algorithm. This may be alleviated with alternative sampling algorithms, but that remains beyond the scope of this paper.

\begin{figure*}[h!]
    \centering
    % \begin{subfigure}[b]{0.3\textwidth}
    %     \includegraphics[width=\textwidth]{KL_VAE.png}
    %     \caption{Regular VAE}
    %     \label{fig:vae_samples}
    % \end{subfigure}
    % ~ %add desired spacing between images, e. g. ~, \quad, \qquad, \hfill etc. 
    %   %(or a blank line to force the subfigure onto a new line)
    % \begin{subfigure}[b]{0.3\textwidth}
    %     \includegraphics[width=\textwidth]{Fisher_VAE.png}
    %     \caption{Fisher VAE}
    %     \label{fig:fisher_samples}
    % \end{subfigure}
    % ~ %add desired spacing between images, e. g. ~, \quad, \qquad, \hfill etc. 
    % %(or a blank line to force the subfigure onto a new line)
    % \begin{subfigure}[b]{0.3\textwidth}
    %     \includegraphics[width=\textwidth]{Test_data.png}
    %     \caption{Test data}
    %     \label{fig:tes_samples}
    % \end{subfigure}
     \includegraphics[scale=0.35]{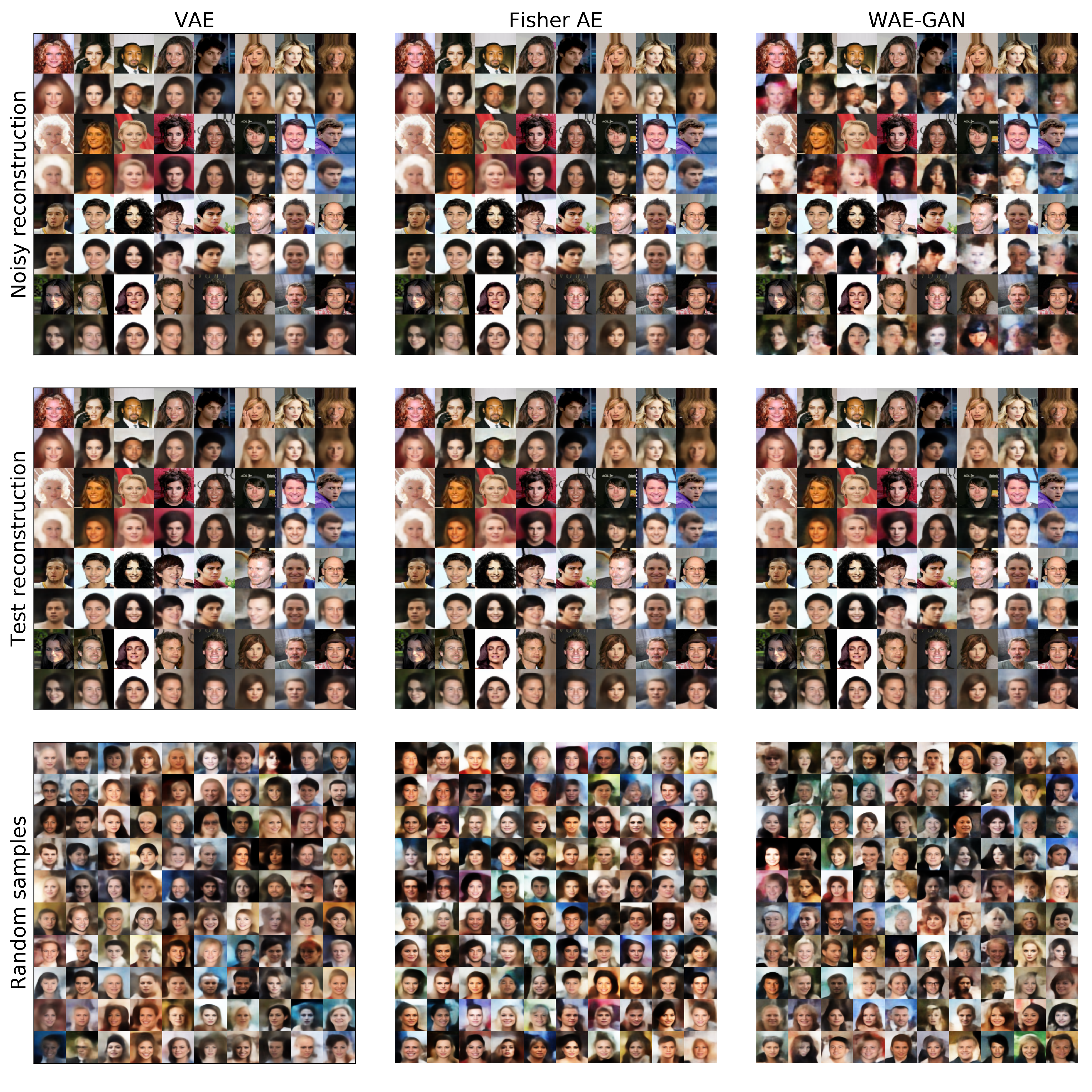}
    \caption{Performance of the Fisher AE trained on celebA dataset in comparison with VAE and WAE-GAN. True test data are given by the odd rows in the reconstruction tasks (rows 1 and 2).}
    \label{fig:agg_celeba_2}
\end{figure*}
\section{Conclusion}
\label{conclusion}
In this paper, we introduced a new type of auto-encoders constructed based on the minimization of the Fisher divergence between the joint distribution over the data and latent variables and the model joint distribution. The resulting loss function has two interesting aspects: 1) it allows to directly minimize the tractable Fisher divergence between the approximate and the true posteriors and 2) considers a stability measure of the encoder that allows to produce robust features. Experimental results were provided to demonstrate the competitive performance of the proposed Fisher auto-encoders as compared to some existing schemes like VAEs and Wasserstein AEs and their superiority in terms of robustness. An interesting but non trivial extension of the present work is to consider the modeling of the posterior distribution using exponential family priors.

\subsubsection*{Acknowledgements}
This work was supported by the Army Research Office grant No. W911NF-15-1-0479.
% All acknowledgments go at the end of the paper, including thanks to reviewers who gave useful comments, to colleagues who contributed to the ideas, and to funding agencies and corporate sponsors that provided financial support. 
% To preserve the anonymity, please include acknowledgments \emph{only} in the camera-ready papers.

%\bibliographystyle{numeric}
% \bibliography{References}
\bibliographystyle{apa}
\bibliography{References}
\onecolumn
\pagebreak
\section*{SUPPLEMENTARY MATERIAL}
\section*{Proof of Theorem 1}
\begin{align*}
    \phi^{\star}, \eta^{\star}, \theta^{\star} & = \arg \min_{\phi, \eta, \theta} \mathbb{D}_{\nabla}    \left[q_{\star, \phi}(\mathbf{x},\mathbf{z}) || p_{\theta}(\mathbf{x}, \mathbf{z}) \right] \\ 
 & =   \arg \min_{\phi, \eta, \theta} \mathbb{E}_{q_{\star, \phi} (\mathbf{x}, \mathbf{z})}  \frac{1}{2} \left \| \nabla_{\mathbf{x}, \mathbf{z}} \log q_{\star, \phi}(\mathbf{x},\mathbf{z}) - \nabla_{\mathbf{x}, \mathbf{z}} \log p_{\eta, \theta}(\mathbf{x},\mathbf{z})\right\|^2 \\
 & = \arg \min_{\phi, \eta, \theta} \mathbb{E}_{p_{\star}(\mathbf{x})}  \mathbb{E}_{q_{\phi}(\mathbf{z}| \mathbf{x})} \frac{1}{2} \left \| \nabla_{\mathbf{x}, \mathbf{z}} \log q_{\star, \phi}(\mathbf{x},\mathbf{z}) - \nabla_{\mathbf{x}, \mathbf{z}} \log p_{\eta, \theta}(\mathbf{x},\mathbf{z})\right\|^2 \\ 
 & = \arg \min_{\phi, \eta, \theta}  \mathbb{E}_{p_{\star}(\mathbf{x})}  \mathbb{E}_{q_{\phi}(\mathbf{z}| \mathbf{x})} \frac{1}{2} \left \| \nabla_{\mathbf{x}} \log p_{\star}(\mathbf{x}) + \nabla_{\mathbf{x}} \log q_{\phi}(\mathbf{z} | \mathbf{x}) - \nabla_{\mathbf{x}} \log p_{ \theta}(\mathbf{x} |\mathbf{z})\right\|^2 \\
 & \quad \quad \quad \quad \quad  + \mathbb{E}_{p_{\star}(\mathbf{x})}  \underbrace{\mathbb{E}_{q_{\phi}(\mathbf{z}| \mathbf{x})} \frac{1}{2} \left \| \nabla_{\mathbf{z}} \log q_{\phi}(\mathbf{z} | \mathbf{x}) - \nabla_{\mathbf{z}} \log p_{\eta, \theta}(\mathbf{z} | \mathbf{x})\right \|^2}_{\mathbb{D}_{\nabla} \left[q_{\phi}(\mathbf{z}| \mathbf{x}) || p_{\eta, \theta}(\mathbf{z} | \mathbf{x}) \right]}  \\
 & = \arg \min_{\phi, \eta, \theta} \mathbb{E}_{p_{\star}(\mathbf{x})} \mathbb{D}_{\nabla} \left[q_{\phi}(\mathbf{z}| \mathbf{x}) || p_{\eta, \theta}(\mathbf{z} | \mathbf{x}) \right] + \mathbb{E}_{p_{\star}(\mathbf{x})}  \mathbb{E}_{q_{\phi}(\mathbf{z}| \mathbf{x})} \frac{1}{2} \left \|\nabla_{\mathbf{x}} \log p_{\star}(\mathbf{x}) \right\|^2  \\
 & \quad \quad \quad \quad \quad + \mathbb{E}_{p_{\star}(\mathbf{x})}  \mathbb{E}_{q_{\phi}(\mathbf{z}| \mathbf{x})} \nabla_{\mathbf{x}} \log p_{\star}(\mathbf{x})^{\top} \nabla_{\mathbf{x}} \log q_{\phi}(\mathbf{z} | \mathbf{x}) \\ 
 & \quad \quad \quad \quad \quad  + \mathbb{E}_{p_{\star}(\mathbf{x})} \mathbb{E}_{q_{\phi}(\mathbf{z}| \mathbf{x})} \frac{1}{2} \left \|\nabla_{\mathbf{x}} \log q_{\phi}(\mathbf{z} | \mathbf{x}) \right\|^2 - \nabla_{\mathbf{x}} \log p_{\theta}(\mathbf{x} | \mathbf{z})^{\top} \nabla_{\mathbf{x}} \log q_{\phi}(\mathbf{z} | \mathbf{x}) \\ 
& \quad \quad \quad \quad \quad  + \mathbb{E}_{p_{\star}(\mathbf{x})} \mathbb{E}_{q_{\phi}(\mathbf{z}| \mathbf{x})} \frac{1}{2}\left \|\nabla_{\mathbf{x}} \log p_{\theta}(\mathbf{x} | \mathbf{z}) \right\|^2 - \nabla_{\mathbf{x}} \log p_{\theta}(\mathbf{x} | \mathbf{z})^{\top} \nabla_{\mathbf{x}} \log p_{\star}(\mathbf{x}) \\ 
\end{align*}
Let's examine the inner-product terms:
\begin{align*}
&  \mathbb{E}_{p_{\star}(\mathbf{x})} \mathbb{E}_{q_{\phi}(\mathbf{z}| \mathbf{x})}   \nabla_{\mathbf{x}} \log p_{\star}(\mathbf{x})^{\top} \nabla_{\mathbf{x}} \log q_{\phi}(\mathbf{z} | \mathbf{x}) \\ & = \iint p_{\star}(\mathbf{x}) q_{\phi}(\mathbf{z}| \mathbf{x}) \nabla_{\mathbf{x}} \log p_{\star}(\mathbf{x})^{\top} \nabla_{\mathbf{x}} \log q_{\phi}(\mathbf{z} | \mathbf{x}) d\mathbf{z} d\mathbf{x} \\
 & = \iint p_{\star}(\mathbf{x}) q_{\phi}(\mathbf{z}| \mathbf{x}) \nabla_{\mathbf{x}} \log p_{\star}(\mathbf{x})^{\top} \nabla_{\mathbf{x}} \log q_{\phi}(\mathbf{z} | \mathbf{x}) d\mathbf{x} d\mathbf{z} \\
 & = \iint p_{\star}(\mathbf{x}) \nabla_{\mathbf{x}} \log p_{\star}(\mathbf{x})^{\top} \nabla_{\mathbf{x}} q_{\phi}(\mathbf{z} | \mathbf{x}) d\mathbf{x} d\mathbf{z} \\ 
 & \stackrel{(a)}{=} - \iint p_{\star}(\mathbf{x}) q_{\phi}(\mathbf{z}| \mathbf{x}) \left[ \left \| \nabla_{\mathbf{x}} \log p_{\star}(\mathbf{x})\right\|^2+ \Delta_{\mathbf{x}} \log p_{\star}(\mathbf{x})  \right] d\mathbf{x} d\mathbf{z} \\
 & = -  \mathbb{E}_{p_{\star}(\mathbf{x})} \mathbb{E}_{q_{\phi}(\mathbf{z}| \mathbf{x})} \left[ \left \| \nabla_{\mathbf{x}} \log p_{\star}(\mathbf{x})\right\|^2+ \Delta_{\mathbf{x}} \log p_{\star}(\mathbf{x})  \right],
\end{align*}
where $(a)$ is obtained by an integration by parts. 
\begin{align*}
  & \mathbb{E}_{p_{\star}(\mathbf{x})} \mathbb{E}_{q_{\phi}(\mathbf{z}| \mathbf{x})}    -  \nabla_{\mathbf{x}} \log p_{\theta}(\mathbf{x} | \mathbf{z})^{\top} \nabla_{\mathbf{x}} \log q_{\phi}(\mathbf{z} | \mathbf{x}) \\ & = - \iint p_{\star}(\mathbf{x}) q_{\phi}(\mathbf{z}| \mathbf{x}) \nabla_{\mathbf{x}} \log p_{\theta}(\mathbf{x} | \mathbf{z})^{\top} \nabla_{\mathbf{x}} \log q_{\phi}(\mathbf{z} | \mathbf{x}) d\mathbf{z} d\mathbf{x} \\
  & = - \iint   p_{\star}(\mathbf{x}) \nabla_{\mathbf{x}} \log p_{\theta}(\mathbf{x} | \mathbf{z})^{\top} \nabla_{\mathbf{x}}  q_{\phi}(\mathbf{z} | \mathbf{x}) d\mathbf{x} d\mathbf{z} \\
  & \stackrel{(b)}{=}  \iint p_{\star}(\mathbf{x}) q_{\phi}(\mathbf{z}| \mathbf{x}) \left[ \Delta_{\mathbf{x}} \log p_{\theta}(\mathbf{x} | \mathbf{z}) + \nabla_{\mathbf{x}} \log p_{\theta}(\mathbf{x} | \mathbf{z})^{\top} \nabla_{\mathbf{x}} \log p_{\star}(\mathbf{x})  \right]  d\mathbf{x} d\mathbf{z} \\
  & = \mathbb{E}_{p_{\star}(\mathbf{x})} \mathbb{E}_{q_{\phi}(\mathbf{z}| \mathbf{x})}  \left[ \Delta_{\mathbf{x}} \log p_{\theta}(\mathbf{x} | \mathbf{z}) + \nabla_{\mathbf{x}} \log p_{\theta}(\mathbf{x} | \mathbf{z})^{\top} \nabla_{\mathbf{x}} \log p_{\star}(\mathbf{x})  \right],
\end{align*}
where $(b)$ is again obtained by an integration by parts. 
Grouping all the terms together, we get 
\begin{align*}
 & \phi^{\star}, \eta^{\star}, \theta^{\star}  \\ & = \arg \min_{\phi, \eta, \theta} \mathbb{E}_{p_{\star}(\mathbf{x})} - s_{\nabla}\left[p_{\star}(\mathbf{x})\right] +  \mathbb{D}_{\nabla} \left[q_{\phi}(\mathbf{z}| \mathbf{x}) || p_{\eta, \theta}(\mathbf{z} | \mathbf{x}) \right]  + \mathbb{E}_{q_{\phi}(\mathbf{z}| \mathbf{x})} s_{\nabla}\left[p_{\theta}(\mathbf{x} | \mathbf{z})\right]  \\ 
 & \quad \quad \quad \quad \quad \quad \quad + \frac{1}{2} \left \|\nabla_{\mathbf{x}} \log q_{\phi}(\mathbf{z} | \mathbf{x}) \right\|^2. 
\end{align*}
By noticing that $\mathbb{E}_{p_{\star}(\mathbf{x})}-s_{\nabla}\left[p_{\star}(\mathbf{x})\right]$ is independent of the parameters $\phi$, $\eta$ and $\theta$, we conclude the proof of Theorem 1. 
% \section{More results}
% \begin{figure}[h!]
%     \centering
%     \includegraphics[scale=0.45]{MNIST_samples.png}
%     \caption{Random samples generated by the Fisher AE with exp. family prior.}\label{fig:mnist}
% \end{figure}
% \vspace*{-14pt}
\section*{Robustness to binary masking noise}
We extend the experiments to examine the robustness of the proposed Fisher AEs and consider another type of noise called binary masking noise which consists on setting the value of a randomly selected fraction $\nu$ of input components to zero. 
\subsection*{MNIST}
\begin{figure}[h!]
      \begin{center}
    \includegraphics[scale=0.44]{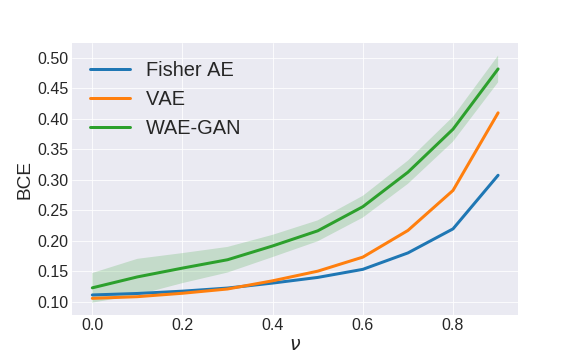}
  \end{center}
    \caption{BCE vs. the fraction $\nu$ for MNIST.}
    \label{fig:BCE_mask_mnist}
\end{figure}
\begin{figure}[h]
    \centering
     \includegraphics[scale=0.47]{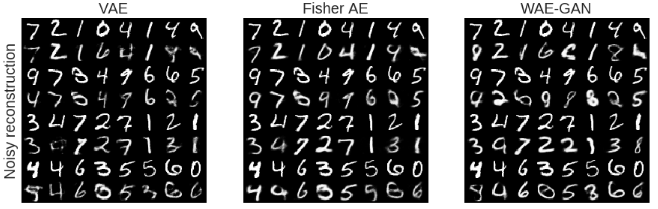}
    \caption{Test reconstruction results when a random fraction $\nu=0.8$ of test data is set to zero. True test data are given by the odd rows.}
    \label{fig:mnist_mask_samples}
\end{figure}
The same insights regarding the robustness of the Fisher AE to Gaussian noise are confirmed in the case of binary masking noise. Both Figures \ref{fig:BCE_mask_mnist} and \ref{fig:mnist_mask_samples} shows the superiority of Fisher AE in terms of robustness to binary masking noise as compared to VAE and WAE-GAN. 
\subsection*{CelebA}
\begin{figure}[h!]
      \begin{center}
    \includegraphics[scale=0.44]{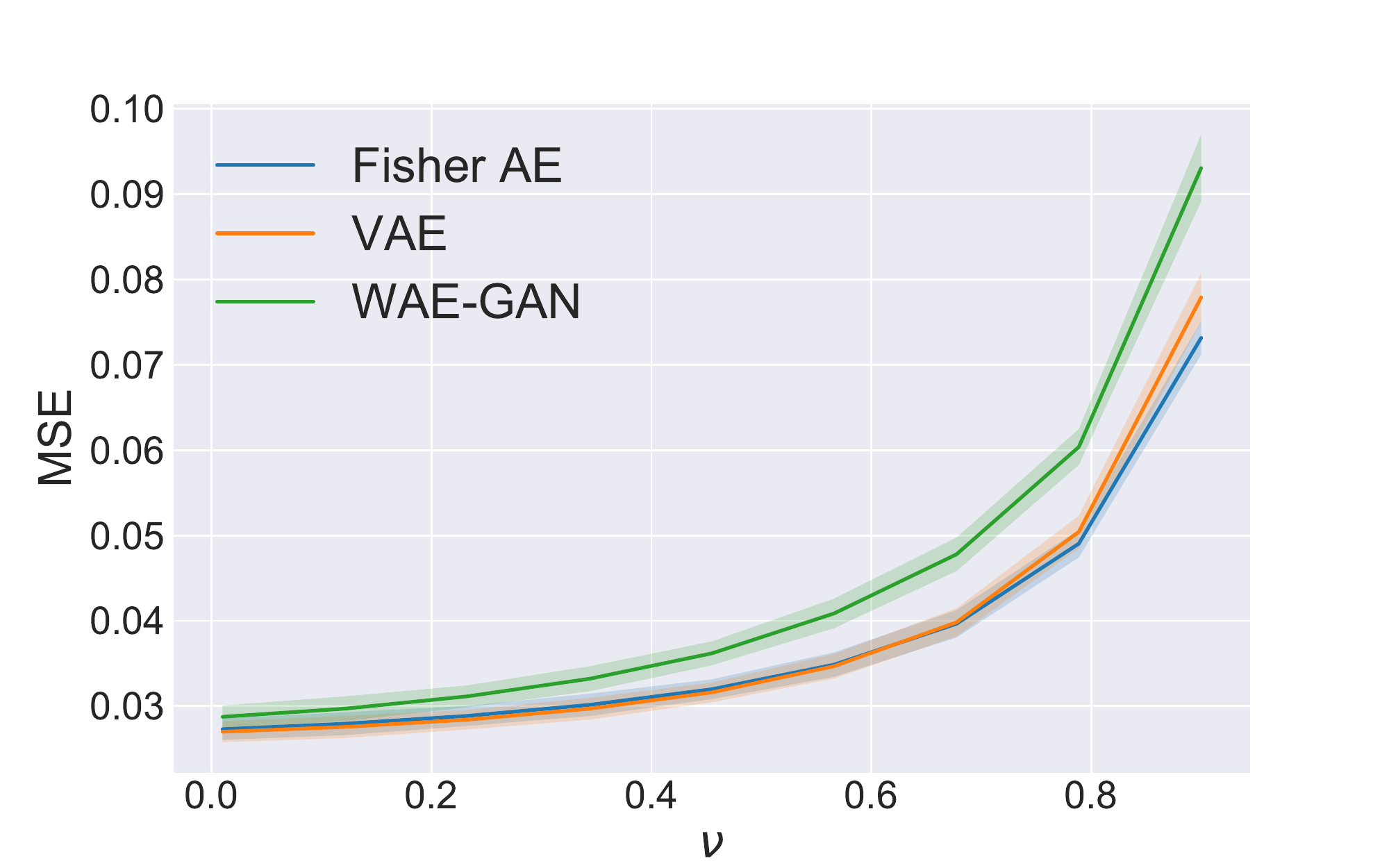}
  \end{center}
    \caption{MSE vs. the fraction $\nu$ for celebA with random mask.}
    \label{fig:celeba_mask}
\end{figure}
\begin{figure}[h]
    \centering
     \includegraphics[scale=0.33]{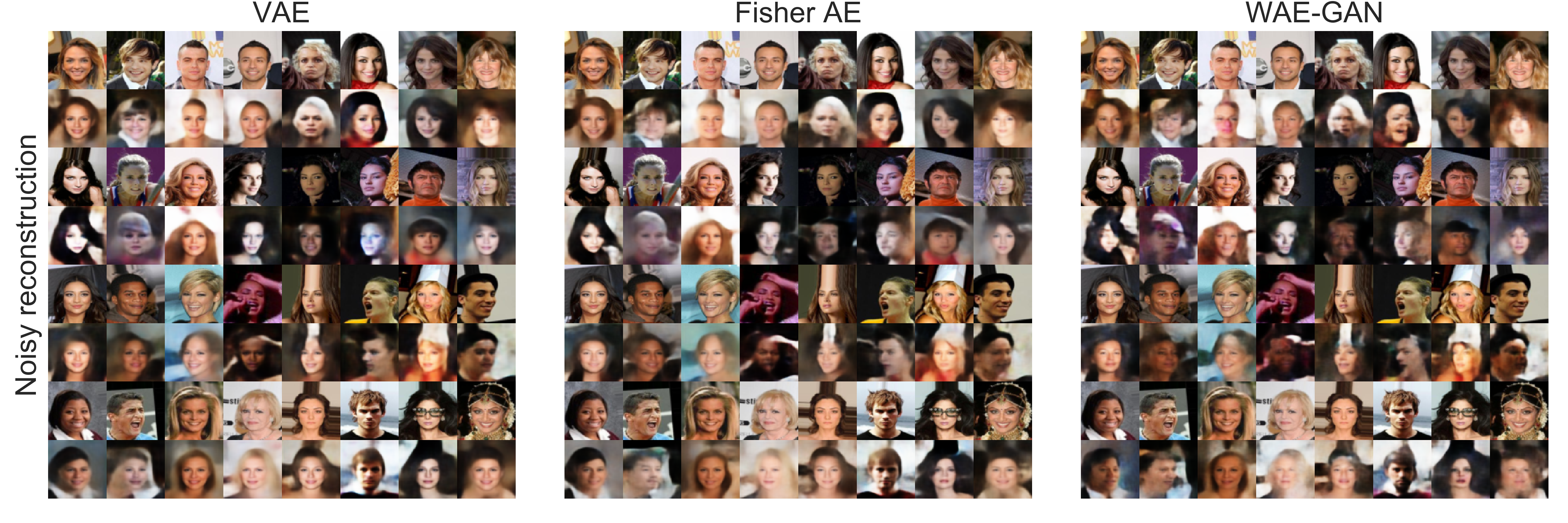}
    \caption{Test reconstruction results when a random fraction $\nu=0.8$ of test data is set to zero. True test data are given by the odd rows.}
    \label{fig:celeba_mask_samples}
\end{figure}
As shown in  Figures \ref{fig:celeba_mask} and \ref{fig:celeba_mask_samples}, in the case of CelebA, both VAE and Fisher AE exhibit similar but superior performance in terms of robustness against binary masking noise as compared to WAE-GAN. 
\section*{Further details on experiments}
% \vspace*{-14pt}
Here, we give the detailed architecture used in the implementation of the different auto-encoders for both MNIST and celebA data sets. 
\begin{itemize}
\item $\text{FC}(n_{in}, n_{out})$: Fully connected layer with input/output dimensions given by $n_{in}$ and $n_{out}$. 
\item $\text{Conv}(n_{in}, n_{out}, k, s, p)$: Convolutional layer with input channels $n_{in}$, output channels $n_{out}$, kernel size $k$, stride $s$ and padding $p$.
\item $\text{ConvT}(n_{in}, n_{out}, k, s, p)$: Transposed convolutional layer with input channels $n_{in}$, output channels $n_{out}$, kernel size $k$, stride $s$ and padding $p$. 
\item $\text{AvgPool}(k, s, p)$: Average Pooling with kernel size, stride and padding respectively given by $k$, $s$ and $p$. 
\item BN : Batch-normalization
\item BiI : 2D bilinear interpolation layer 
\end{itemize}
%  \vspace*{-15pt}
\pagebreak
\subsection*{MNIST}
%\subsubsection{Encoder/Decoder architectures for Fisher AE and VAE}
\begin{table}[th!]
\caption{Encoder/Decoder architectures for Fisher AE and VAE for MNIST}
\centering
\begin{tabular}{|l|l|}
\hline
\textbf{Encoder}                                                                                                                                                                                                                          & \textbf{Decoder}                                                                                                                                                                                                                       \\ \hline \hline
\begin{tabular}[c]{@{}l@{}}Input size: (1, 28, 28)\\ Conv(1, 64, 3, 2, 2) \\ LeakyReLU\\ Conv(64, 128, 3, 2, 2) \\ BN LeakyReLU\\ Conv(128, 256, 3, 2, 2) \\ BN, LeakyReLU\\ Conv(256, 512, 3, 2, 1)\\ BN, LeakyReLU \\ Conv(512, 16, 3, 2, 0) \\ Output size : (16, 1, 1)\end{tabular} & \begin{tabular}[c]{@{}l@{}}Input size: (8,  1,  1)\\ ConvT(8, 512, 5, 1, 1) \\ BN, ReLU \\ ConvT(512, 256, 5, 1, 1) \\ BN, ReLU\\ ConvT(256, 128, 5, 2, 1) \\ BN, ReLU\\ ConvT(128, 64, 5, 1, 1) \\ BN, ReLU\\ ConvT(64, 1, 4, 2, 0) \\ Sigmoid \\ Output size : (1, 28, 28) \end{tabular} \\ \hline
\end{tabular}
\end{table}
%\subsubsection{Encoder/Generator/Discriminator architectures for Wasserstein AE for MNIST}
\begin{table}[h!]
\caption{Encoder/Generator/Discriminator architectures for WAE-GAN for MNIST}
\centering
\begin{tabular}{|l|l|l|}
\hline
\textbf{Encoder}                                                                                                                                                                                                                                                                         & \textbf{Generator}                                                                                                                                                                                                                                                                    & \textbf{Discriminator}                                                                                                                                 \\ \hline \hline 
\begin{tabular}[c]{@{}l@{}}Input size : (1, 28, 28)\\ Conv(1, 64, 3, 2, 2) \\ LeakyReLU\\ Conv(64, 128, 3, 2, 2) \\ BN, LeakyReLU\\ Conv(128, 256, 3, 2, 2) \\ BN, LeakyReLU\\ Conv(256, 512, 3, 2, 1) \\ BN, LeakyReLU\\ Conv(512, 8, 3, 1, 0) \\ Output size : (8, 1, 1)\end{tabular} & \begin{tabular}[c]{@{}l@{}}Input size : (8, 1, 1)\\ ConvT(8, 512, 5, 1, 1) \\ BN, ReLU\\ ConvT(512, 256, 5, 1, 1) \\ BN, ReLU\\ ConvT(256, 128, 5, 2, 1) \\ BN, ReLU\\ ConvT(128, 64, 5, 1, 1) \\ BN, ReLU\\ ConvT(64, 1, 4, 2, 0) \\ Sigmoid\\ Output size : (1, 28, 28)\end{tabular} & \begin{tabular}[c]{@{}l@{}}Input size : (8, 1, 1)\\ Flatten\\ FC(8, 256) \\ ReLU\\ FC(256, 1) \\ Sigmoid\\ \\ \\ Output size : (1, )\end{tabular} \\ \hline
\end{tabular}
\end{table}
\pagebreak
\subsection*{CelebA}
%\subsubsection{Encoder/Decoder architectures for Fisher AE and VAE}
\begin{table}[h!]
\caption{Encoder/Decoder architectures for Fisher AE and VAE for CelebA}
\centering
\begin{tabular}{|l|l|}
\hline
\textbf{Encoder}                                                                                                                                                                                                                          & \textbf{Decoder}                                                                                                                                                                                                                       \\ \hline \hline
\begin{tabular}[c]{@{}l@{}}Input size: (3, 64, 64)\\ Conv(3, 64, 5, 1, 2) \\ LeakyReLU, AvgPool(2, 2, 0)\\ Conv(64, 128, 5, 1, 2) \\ BN, LeakyReLU, AvgPool(2, 2, 0)\\ Conv(128, 256, 5, 1, 2) \\ BN, LeakyReLU, AvgPool(2, 2, 0)\\ Conv(256, 512, 5, 1, 2) \\ BN, LeakyReLU\\ Flatten \\ FC(8192, 64) \\ Output size : (64, )\end{tabular} & \begin{tabular}[c]{@{}l@{}}Input size: (64,  1,  1)\\ BiI, Conv(64, 512, 5, 1, 02) \\ BN, ReLU\\ BiI, Conv(512, 256, 5, 1, 2) \\ BN, ReLU\\ BiI, Conv(256, 128, 5, 1, 2) \\ BN, ReLU\\ BiI, Conv(128, 64, 5, 1, 2) \\ BN, ReLU\\ BiI, Conv(64, 3, 5, 1, 2) \\ Tanh \\ Output size : (3, 64, 64) \end{tabular} \\ \hline
\end{tabular}
\end{table}

%\subsubsection{Encoder/Generator/Discriminator architectures for Wasserstein AE}
\begin{table}[h!]
\caption{Encoder/Generator/Discriminator architectures for WAE-GAN for CelebA}
\centering
\begin{tabular}{|l|l|l|}
\hline
\textbf{Encoder}                                                                                                                                                                                                                                                                         & \textbf{Generator}                                                                                                                                                                                                                                                                    & \textbf{Discriminator}                                                                                                                                 \\ \hline \hline 
\begin{tabular}[c]{@{}l@{}}Input size: (3, 64, 64)\\ Conv(3, 64, 5, 1, 2) \\ LeakyReLU, AvgPool(2, 2, 0)\\ Conv(64, 128, 5, 1, 2) \\ BN, LeakyReLU, AvgPool(2, 2, 0)\\ Conv(128, 256, 5, 1, 2) \\ BN, LeakyReLU, AvgPool(2, 2, 0)\\ Conv(256, 512, 5, 1, 2) \\ BN, LeakyReLU\\ Flatten \\ FC(8192, 64) \\ Output size : (64, )\end{tabular} & \begin{tabular}[c]{@{}l@{}}Input size: (64,  1,  1)\\ BiI, Conv(64, 512, 5, 1, 02) \\ BN, ReLU\\ BiI, Conv(512, 256, 5, 1, 2) \\ BN, ReLU\\ BiI, Conv(256, 128, 5, 1, 2) \\ BN, ReLU\\ BiI, Conv(128, 64, 5, 1, 2) \\ BN, ReLU\\ BiI, Conv(64, 3, 5, 1, 2) \\ Tanh \\ Output size : (3, 64, 64)\end{tabular} & \begin{tabular}[c]{@{}l@{}}Input size : (8, 1, 1)\\ Flatten\\ FC(8, 256) \\ ReLU\\ FC(256, 1) \\ Sigmoid\\ \\ \\ Output size : (1, )\end{tabular} \\ \hline
\end{tabular}
\end{table}
% If you need to include additional appendices during submission, you can include them in the supplementary material file.
% You can submit a single file of additional supplementary material which may be either a pdf file (such as proof details) or a zip file for other formats/more files (such as code or videos). 
% Note that reviewers are under no obligation to examine your supplementary material. 
% If you have only one supplementary pdf file, please upload it as is; otherwise gather everything to the single zip file.

% You must use \texttt{aistats2021.sty} as a style file for your supplementary pdf file and follow the same formatting instructions as in the main paper. 
% The only difference is that it must be in a \emph{single-column} format.
% You can use \texttt{supplement.tex} in our starter pack as a starting point.
% Alternatively, you may append the supplementary content to the main paper and split the final PDF into two separate files.

\end{document}